\documentclass[a4paper, 11pt,headings=small,abstract]{scrartcl}
\usepackage[english]{babel}

\usepackage{amscd,amsmath,amsthm,array,hhline,mathrsfs, enumerate, booktabs,fancybox,calc,textcomp,xcolor,graphicx,bbm,xspace,nicefrac,stmaryrd,url,arcs,listings,amssymb}
\usepackage[title]{appendix}
\usepackage[citestyle = numeric-comp,bibstyle = numeric, firstinits=true, natbib = true, backend = bibtex,maxbibnames=6, url=false, isbn=false, sortcites=true]{biblatex}
\usepackage[semibold]{libertine}
\usepackage[upint,amsthm]{libertinust1math}
\usepackage[scr=boondoxupr, scrscaled = 1.0, cal =stixtwoplain]{mathalpha}
\usepackage[scaled=0.95]{inconsolata}
\usepackage{accents}
\usepackage{dsfont}
\DeclareMathAlphabet{\mathsf}{OT1}{\sfdefault}{m}{n}
\SetMathAlphabet{\mathsf}{bold}{OT1}{\sfdefault}{b}{n}
\DeclareMathAlphabet{\mathfrak}{U}{jkpmia}{m}{it}
\SetMathAlphabet{\mathfrak}{bold}{U}{jkpmia}{bx}{it}
\usepackage[utf8]{inputenc}
\usepackage{enumitem}
\usepackage{etoolbox}
\usepackage[normalem]{ulem}
\usepackage{mathtools}
\usepackage{geometry}
\usepackage{float}

\usepackage{bm}
\usepackage{tikz}
\usepackage[outdir =./]{epstopdf}
\numberwithin{equation}{section}
\RequirePackage{algorithm}
\RequirePackage{algorithmic}

\usepackage{chngcntr}
\counterwithin{figure}{section}

\usepackage[dvipsnames]{xcolor}
\usepackage[colorlinks,linkcolor=RedViolet,citecolor=myteal,anchorcolor=Green,urlcolor=BlueViolet]{hyperref}

\definecolor{WIMgreen}{RGB}{60 134 132}
\definecolor{UMblue}{RGB}{4 47 86}
\definecolor{myteal}{RGB}{0 123 137}
\definecolor{material_green}{RGB}{27 43 52}
\definecolor{dracula_pink}{RGB}{180 93 149}
\definecolor{dracula_blue}{RGB}{40 42 54}
\definecolor{dracula_turq}{RGB}{92 143 159}
\definecolor{dracula_orange}{RGB}{255 184 108}
\definecolor{material_petrol}{RGB}{2 119 189}
\definecolor{Purple}{RGB}{103 58 183}

\newcommand{\e}{\mathrm{e}}
\usepackage{verbatim}
\usepackage[title]{appendix}

\definecolor{refkey}{gray}{.5}
\definecolor{labelkey}{gray}{.5}
\setkomafont{sectioning}{\normalcolor\bfseries}
\setkomafont{descriptionlabel}{\normalcolor\bfseries}
\setkomafont{section}{\large}
\setkomafont{subsection}{\normalsize}
\setkomafont{subsubsection}{\normalsize}
\setkomafont{paragraph}{\normalsize}
\setkomafont{subparagraph}{\normalsize}
\setkomafont{author}{\large}
\setkomafont{title}{\large\selectfont\bfseries}
\setkomafont{date}{\normalsize}

\usepackage{subcaption}

\theoremstyle{plain}
\newtheorem{theorem}{Theorem}
\newtheorem{lemma}[theorem]{Lemma}

\newtheorem{proposition}[theorem]{Proposition}

\theoremstyle{definition}
\newtheorem{remark}[theorem]{Remark}

\newtheorem{innercustomgeneric}{\customgenericname}
\providecommand{\customgenericname}{}
\newcommand{\newcustomtheorem}[2]{%
  \newenvironment{#1}[1]
  {%
   \renewcommand\customgenericname{#2}%
   \renewcommand\theinnercustomgeneric{##1}%
   \innercustomgeneric
  }
  {\endinnercustomgeneric}
}

\newcustomtheorem{customass}{Assumption}

\numberwithin{equation}{section}
\numberwithin{theorem}{section}

\DeclareMathOperator*{\argmin}{arg\,min}

\DeclareMathOperator{\PP}{{\mathbb P}}
\DeclareMathOperator{\E}{{\mathbb E}}
\DeclareMathOperator{\R}{{\mathbb R}}

\newcommand*\diff{\mathop{}\!\mathrm{d}}

\newcommand{\one}{\bm{1}}
\renewcommand{\hat}{\widehat}
\renewcommand{\tilde}{\widetilde}
 
\newcommand{\supp}{\operatorname{supp}}
\newcommand{\Ex}{\mathbb{E}}

\newcommand{\state}{\mathcal I}
\newcommand{\PR}{\mathbb{P}}
\newcommand{\coloneqq}{\coloneq}

\makeatletter
\DeclareRobustCommand{\cev}[1]{%
  \mathpalette\do@cev{#1}%
}
\newcommand{\do@cev}[2]{%
  \fix@cev{#1}{+}%
  \reflectbox{$\m@th#1\vec{\reflectbox{$\fix@cev{#1}{-}\m@th#1#2\fix@cev{#1}{+}$}}$}%
  \fix@cev{#1}{-}%
}
\newcommand{\fix@cev}[2]{%
  \ifx#1\displaystyle
    \mkern#23mu
  \else
    \ifx#1\textstyle
      \mkern#23mu
    \else
      \ifx#1\scriptstyle
        \mkern#22mu
      \else
        \mkern#22mu
      \fi
    \fi
  \fi
}

\makeatother

\numberwithin{equation}{section}
\numberwithin{theorem}{section}
\allowdisplaybreaks

\bibliography{example_paper}

\title{\fontsize{16}{19} \selectfont Beyond Fixed Horizons: A Theoretical Framework for Adaptive Denoising Diffusions}
\author{Sören Christensen\thanks{Kiel University, Department of Mathematics, Kiel, Germany\newline Email: \href{mailto:christensen@math.uni-kiel.de}{christensen@math.uni-kiel.de}} \and Jan Kallsen\thanks{Kiel University, Department of Mathematics, Kiel, Germany\newline Email: \href{mailto:kallsen@math.uni-kiel.de}{kallsen@math.uni-kiel.de}}  \and Claudia Strauch\thanks{Heidelberg University, Institute of Mathematics, Heidelberg, Germany\newline Email: \href{mailto:strauch@math.uni-heidelberg.de}{strauch@math.uni-heidelberg.de}} \and Lukas Trottner\thanks{University of Stuttgart, Institute for Stochastics and Applications, Stuttgart, Germany. \newline Email: \href{mailto:lukas.trottner@mathematik.uni-stuttgart.de}{lukas.trottner@mathematik.uni-stuttgart.de}}}
\date{\vspace{-15pt}}

\begin{document} 
\maketitle
\begin{abstract}
We introduce a new class of generative diffusion models that, unlike conventional denoising diffusion models, achieve a time-homogeneous structure for both the noising and denoising processes, allowing the number of steps to adaptively adjust based on the noise level. This is accomplished by conditioning the forward process using Doob's $h$-transform, which terminates the process at a suitable sampling distribution at a random time. The model is particularly well suited for generating data with lower intrinsic dimensions, as the termination criterion simplifies to a first-hitting rule. A key feature of the model is its adaptability to the target data, enabling a variety of downstream tasks using a pre-trained unconditional generative model. These tasks include natural conditioning through appropriate initialisation of the denoising process and classification of noisy data.
\end{abstract}

\section{Introduction}
\label{sec:introduction}
Denoising Diffusion Models \cite{ho20, song21} have gained significant attention in recent years due to their ability to generate high-quality data samples by iteratively denoising simple distributions based on the learned dynamics of a time-reversed forward noising process initalised in the target data distribution \cite{niu20, dhariwal21, ho22, watson23, yang23, evans24}. A key limitation of these models, however, is their reliance on a fixed time horizon, which introduces an artificial time dependency in the drift function of the backward process. As a result, the generative denoising process follows a predefined number of steps, regardless of the actual level of noise present along the generated path.

To overcome this limitation, we introduce a novel class of diffusion models that dynamically adapt to the state of the denoising process. By replacing the fixed deterministic time horizon with a random one and conditioning the behaviour of the forward process at a terminal time, our approach achieves greater flexibility and state awareness. The foundation of our method lies in Doob's $h$-transforms with respect to underlying exponential times. While the theoretical groundwork for this concept exists, its explicit application and detailed exploration -- particularly in comparison to deterministic time horizons -- remains underrepresented in the literature.

A key feature of our model is its inherent adaptability: the number of denoising steps dynamically adjusts based on the noise level in the data, introducing a stochastic element. This randomness not only has the potential to enhance the generation process, but also allows denoising to start from partially noisy data, naturally incorporating conditioning. Moreover, the time required for denoising serves as an intuitive measure of the distance between noisy observations and the underlying data distribution, providing the basis for tasks such as classification and anomaly detection. The model's architecture also supports natural conditioning mechanisms, allowing seamless adaptation to diverse tasks without the need for task-specific design modifications.

Thanks to its flexible, time-homogeneous structure, our model offers a fresh perspective on generative tasks that enhances the adaptability and versatility of diffusion models and establishes a robust foundation for transfer learning. 

\paragraph*{Contributions and structure}
Let us briefly summarize our main contributions and the structure of the paper: after discussing related work in Section \ref{sec:related_work}, we first present theoretical results on the $h$-transform and time reversal for exponential time horizons in Section \ref{sec:theory}. Building on this foundation, we develop a universal and flexible diffusion model, accompanied by the associated learning theory, in Section \ref{sec:model}. In this framework, we identify the \textit{polarity} of the data distribution as a central assumption for successful learning, offering a new perspective on the manifold hypothesis. Finally, we discuss the adaptability of the model across various application domains through transfer learning in\ Section \ref{subsec:features}. 

\section{Related work}
\label{sec:related_work}
This section reviews key works relevant to extending diffusion models beyond conventional time dependencies, incorporating elements such as random horizons, Doob's $h$-transform, and conditioning. 

\citet{ye22} introduce first hitting diffusion models, which use first hitting times to capture the intrinsic geometry of data manifolds. Technically, their approach is related to ours in that it uses a random time horizon (in their case via first hitting times) and makes use of Doob's $h$-transform. However, their method defines a backward process that ensures that the generated data lies on a predetermined manifold. In contrast, our approach inverts this perspective and uses the forward process to create a more flexible generative framework that is not constrained to predefined manifolds. Instead, the data manifold is dynamically learned. 

\citet{bortoli21} explore diffusion generative modeling through the lens of Schrödinger Bridges (SB) and solving transport problems. Unlike traditional methods that require running forward SDEs over long durations, SB techniques generate samples in finite time. However, their approach is limited to a deterministic time horizon, where there are well-established connections between Schrödinger Bridges and $h$-transforms. Further developments in this direction are also presented in \citet{peluchetti23} and \citet{shi23}. In contrast, our framework accommodates random time horizons, allowing for both finite and long durations. 

The paper \citet{zhao2024conditional} provides a comprehensive review of approaches to conditional sampling within generative diffusion models. It discusses methods that rely on joint distributions or pre-trained marginal distributions with explicit likelihoods to generate samples conditioned on certain information, addressing challenges in areas such as Bayesian inverse problems. In these approaches, the original unconditional processes are modified in various ways to introduce conditionality. Similarly, \citet{didi2023framework} unifies conditional training and sampling within a common framework based on the $h$-transform. However, the reliance on a fixed time horizon leads to notable differences from the approach presented in our work.

\citet{berner24} introduce an optimal control perspective on denoising diffusion models, demonstrating how the backward generative process naturally emerges as a solution to a suitable optimal control problem. Proposition \ref{prop:stoch_control} establishes a similar connection in our model.

\citet{denker24} use $h$-transform techniques to introduce fine-tuning of a pre-trained diffusion model for conditional data generation. Additionally to the natural conditioning aspect that our model provides without second stage training, we demonstrate in Section \ref{sec:fine_tuning} how an analogous procedure can be implemented for structured conditioning purposes in our modeling pipeline.

Finally, the companion paper \citet{christensen25} provides a detailed high-dimensional  mathematical analysis of the natural conditioning property of time-reversals of exponentially killed Brownian motions as a particular sub-class of the general class of generative models that we introduce, highlighting the immediate usefulness of our model beyond pure unconditional data generation purposes.

\section{Doob's $h$-transform and time-inversion from random times}\label{sec:theory}
Doob's $h$-transform is a versatile mathematical technique for adjusting the dynamics of a stochastic process. It formalizes the concept of conditioning the process on specific events occurring at a random time. This section outlines the key results relevant to our approach. For a more comprehensive discussion, including relevant literature and the proofs, readers are encouraged to consult the Appendix.

\subsection{Doob's $h$-transform for random time horizons}
In our approach, we will employ $h$-transforms of an underlying non-degenerate and symmetric $d$-dimensional diffusion process 
\[
	dZ_t=b(Z_t)\, dt+\sigma(Z_t)\, dW_t,
\]
to describe both the forward and backward process. Typically, $Z$ will be a Brownian motion or a symmetric Ornstein--Uhlenbeck process on $\state=\mathbb{R}^d$, but other self-dual processes are also possible.
The unkilled version of $Z$ has the $r$-Green kernel
\[G_r(x,y)\coloneqq\int_0^\infty \e^{-rt}p_t(x,y)\,dt,\]
where $p_t(x,y)$ are the symmetric transition densities of $Z$ with respect to a reference measure $m$. The reference measure \( m \) may, but does not necessarily, coincide with the standard Lebesgue measure. Up to scaling, $G_r$ describes the pdf of $Z$ at an independent exponential time. For a Brownian motion, for example, $G_r$ can be found explicitly, see Remark \ref{ex:G_r_BM}.

To specify the idea of conditioning at a random time, we introduce two major modifications for $Z$: 
\begin{enumerate}
    \item The drift is modified to attract the process towards the desired states. 
    \item A random stopping time $\zeta$, referred to as the \emph{lifetime}, is introduced to terminate the process at the point where the desired distribution is attained.
\end{enumerate}
Both modifications are seamlessly implemented using the $h$-transform. The $h$-transform is constructed using an $r$-excessive function $h$, which can be defined via a probability measure $\beta$ on the state space $\state$ and a reference state $x_0 \in \state$. To condition the process $Z$ to be distributed according to $\beta$ at its killing time $\zeta$ when initiated from $x_0$, we define
\begin{align*}
	h(x)\coloneqq\int \frac{G_r(x,y)}{G_r(x_0,y)}\, \beta(dy)=\int {G_r(x,y)}\,\kappa(dy),
\end{align*}
where $\kappa(dy)=\kappa_{x_0,\beta}(dy)=\frac{1}{G_r(x_0,y)}\beta(dy)$ is called \emph{representing measure} of $h$.  
Based on this, the $h$-transformed process $Z^h$ is now a Markov process defined by the transition kernel
\[\PR_x(Z^h_t\in dy)=\Ex_x\left[\frac{\e^{-rt}h(Z_t)}{h(x)}\mathbf{1}_{\{Z_t\in dy\}}\right].\]
The essential properties of $Z^h$ are as follows: 
\begin{proposition}\label{prop:properties_h-transf}
\begin{enumerate}
    \item $Z^h$ is killed at some random time $\zeta$, the lifetime, which is an a.s.\ finite random variable.
    \item Outside the support of $\beta$, $Z^h$ is an It\^o diffusion with dynamics
	\begin{align*}
dZ^h_t&=b^h(Z^h_t)\,dt+\sigma(Z^h_t)\,dW_t,\\
b^h(y)&=b(y)+\sigma(y)\sigma(y)^\top\nabla\log h(y).
	\end{align*}   
     \item The distribution of\ $Z^h$ at its lifetime is supported in the support of $\beta$ and given by
		\[
		    \PR_x(Z^h_{\zeta-}\in dy)=\frac{G_r(x,y)}{h(x)}\kappa(dy).
		\]
        In particular, 		\[
		    \PR_{x_0}(Z^h_{\zeta-}\in dy)={\beta}(dy).
		\]
\end{enumerate}
\end{proposition}
The proof is deferred to Appendix \ref{app:doob}.

\subsection{Optimality properties of the $h$-transform}
The $h$-transform is characterised by inherent optimality properties which underline its role as a canonical approach. We now demonstrate that it emerges as the solution to a stochastic stopping and control problem. In this context, the running cost component can be interpreted as minimizing the expected lifetime of the process while penalizing large drifts.
By framing the $h$-transform within this stochastic control problem, we establish a connection to the KL divergence.

\begin{proposition}\label{prop:stoch_control}
    We write $k(u)=r+\frac{1}{2}\|u\|^2$, $g(y)=-\log h(y)$. For an (admissible) control $u$ consider the controlled process $$dZ^u_t=(b(Z^u_t)+\sigma(Z^u_t)u_t)\,dt+\sigma(Z^u_t)\,dW_t,\;t< \tau_\beta,$$ 
    where $\tau_\beta=\tau_\beta^u$ denotes the first entrance time into the support of $\beta$. The stochastic stopping and control problem of minimizing 
\begin{align}\label{eq:control}
J(u,\tau,x) \coloneqq \Ex_x\left[\int_0^\tau k(u_t)\,dt+g(Z^u_{\tau-})\right],\;\qquad v(x)=\inf_{u,\;\tau\le\tau_\beta^u} J(u,\tau,x)
\end{align} 
in both $u$ and $\tau\leq \tau_\beta$ is solved as follows:
\begin{enumerate}
\item $v(x)=g(x)=-\log h(x)$.
\item All pairs $(u^*,\tau)$ for the Markov feedback control $u^*$ given by
\[
u^*(x)\;=\;-\sigma(x)^\top \nabla g(x),
\] 
and arbitrary stopping times $\tau\leq \tau_\beta$ are optimal.
\item For any admissible $u$,
\[
J(u,\tau,x)-v(x)
=\mathrm{KL}\!\left(\PR_x^{Z^u}\big|_{\mathcal F_{\tau_\beta}} \,\Big\Vert\, \PR_x^{Z^h}\big|_{\mathcal F_{\tau_\beta}}\right),
\]
so the Kullback--Leibler divergence between the controlled law and the $h$-transformed law (up to $\tau_\beta$) equals the variational gap.
\end{enumerate}
\end{proposition}

The proof can be found in Appendix \ref{app:stoch_control}. We note here that for the cases we are interested in later, the first entrance time into the support just happens to coincide with the lifetime $\zeta$.

\subsection{Connection to time-inversion}\label{subsec:time_reversal}
Compared to the deterministic case, the time reversal from a random time is much less elaborated in the literature. In\ terms of the $h$-transform, however, the results are very clear: for the $h$-transformed process $Z^h$ with (finite) lifetime $\zeta$, we consider the time-reversed process
\begin{align*}
    \cev{Z^h_s} \coloneqq
        Z^h_{\zeta-s},&\;0<s<\zeta,
\end{align*}
killed at $s=\zeta$.
Let also $\alpha$  be a fixed initial distribution of $Z^h$,
and define
\begin{align}\label{eq:back_h}
    \cev{h}(x)=\int\frac{G_r(x,y)}{h(y)}\,\alpha(dy).
\end{align}
We then have the following key result, whose proof is given in Appendix \ref{app:doob}.
\begin{proposition}\label{prop:time_inversion}
\begin{enumerate}
    \item $\cev{Z^h}$ has the same distribution as $Z^{\cev{h}}$ with initialisation $Z^{\cev{h}}_0 \sim Z^h_{\zeta-}$; in particular,
    	\begin{align*}
d\cev{Z^h_s}&=\cev{b^h}(\cev{Z^h_s})\, ds+\sigma(\cev{Z^h_s})\, d\cev W_s,\\\cev{b^h}(y)&=b(y)+\sigma(y)\sigma(y)^\top\nabla\log \cev h(y),
	\end{align*}
    $\cev W$ a Brownian motion, outside the support of $\alpha$.
    \item $\cev{Z^h}$ is killed on the support of $\alpha$.
    \item When started in $x \notin \supp \alpha$, the distribution of  $Z^{\cev{h}}$ at its lifetime $\zeta$ is given by
\begin{align*}
    \PR_x(Z^{\cev{h}}_{\zeta-}\in dy)=\frac{G_r(x,y)}{\cev{h}(x)h(y)}\,\alpha(dy).
\end{align*}
In particular, for $\PP_x(Z^h_{\zeta-} \in \cdot)$-a.e.\ $x \notin \supp \alpha$ it holds for the time reversed process $\cev{Z^h}$ that 
\begin{align*}
    \PR_x(\cev{Z^h}_{\hspace{-5pt}\zeta-}\in dy)=\frac{G_r(x,y)}{\cev{h}(x)h(y)}\,\alpha(dy).
\end{align*}
\end{enumerate}

\end{proposition}
The interpretation of the time reversal in terms of the $h$-transform has interesting consequences. Let us recall the concept of a polar set: for $Z^h$, $A\subseteq \state$ is called \emph{polar} if $\PR_x(Z_t^h\in A\mbox{ for some }t>0 )=0$ for all $x\in\state$. Typical examples of polar sets for diffusions are sets on low-dimensional submanifolds. For instance, if $Z^h$ is a $d$-dimensional Brownian motion, any set with Hausdorff dimension less than $d-2$ is polar, see e.g. \citet{morters2010brownian}; a related result for general diffusions is also provided in \citet{rama88}.

\begin{proposition}\label{prop:polar}
Assume that the support of the initial distribution $\alpha$ is polar for $Z^h$. Then, $\cev{Z^h}$ is killed at the first entry into the support of $\alpha$. In particular, the distribution of $\cev{Z^h}$ is entirely characterised by the drift $\cev{b^h}$ and the support of $\alpha$.
\end{proposition}
\begin{proof}
The assumption implies that the starting distribution of ${Z^h}$ lies in a polar set, which is not visited during the rest of the lifetime. This means that the time-reversed process $\cev{Z^h}$ only visits this set at the time of killing.
\end{proof}

\subsection{Key examples}\label{subsection:key_examples}
\begin{enumerate}
    \item In the case $\beta(dy)=rG_r(x_0,y)\,m(dy)$ (the distribution of\ $Z$ at the exponential time when started in $x_0$), we have
\[
	h(x)=\int \frac{G_r(x,y)}{G_r(x_0,y)}\,\beta(dy)=\int r{G_r(x,y)}\,m(dy)=1.
\]
Then, $Z^h$ has the same dynamics as $Z$ and $\zeta$ is an independent $\mbox{Exp}(r)$-time. The time-reversed dynamics are determined by $
    \cev{h}(x)=\int{G_r(x,y)}\,\alpha(dy)$.
\item In the case $\beta(dy)=\delta_{x_1}$ for some $x_1\in\state$, we have
\[
	h(x)={G_r(x,x_1)}/c_{x_1},\;c_{x_1}={G_r(x_0,x_1)}.
\]
$h$ is thus independent of the starting point, except for a multiplicative constant that is irrelevant for the $h$-transform, so we may choose $c_{x_1}=1$ w.l.o.g.\ $Z^h$ is thus an exponential bridge killed in $x_1$, no matter where the starting point was, with drift
\[b_h(x)=b(x)+\sigma(x)\sigma(x)^\top\nabla\log  G_r(x,x_1).\]
In this case, 
$$  \cev{h}(x)=\int\frac{G_r(x,y)}{G_r(x_1,y)}\,\alpha(dy).$$
\item In case $Z=W$ is a Brownian motion, we consider radially symmetric functions $h(x)=f(|x|)$ of the form
\[h(x)=\int_{\partial B_R}G_r(x,y)\,\sigma_R(dy),\]
where $\sigma_R$ denotes the surface measure on the sphere with radius $R$.  We obtain that (except for scaling)
\[f(y)=y^{-\nu}I_{\nu}(y\sqrt{2r}),\;\;\nu=\frac{d-2}{2},\]
$I_{\nu}$ is a modified Bessel function of the first kind,
and this results in the forward process $Z^h$ with drift
\[b^h(x)=\nabla\log h(x)=\sqrt{2r}\frac{I_{\nu+1}(\lVert x\rVert\sqrt{2r})}{I_{\nu}(\lVert x \rVert\sqrt{2r})}\frac{x}{\lVert x \rVert},\]
killed when exiting $B_R$ (in the sense of a last exit).
For large dimensions $d$ and moderate values of\ $\lVert x \rVert$, it holds that
$b^h(x)\approx\frac{2r}{d}x$ (see 10.41 in \cite{NIST:DLMF}),
so that the forward process $Z^h$ can be approximated using the (non-stationary) Ornstein--Uhlenbeck process dynamics
\begin{align}\label{eq:OU-approx}
    d\tilde Z_t=\frac{2r}{d}\tilde Z_t\,  dt+dW_t.
\end{align}

\end{enumerate}

\section{The proposed model}\label{sec:model}
Standard diffusion generative models rely on a deterministic time horizon, leading to time-dependent backward processes and inefficiencies in the forward process due to excessive noise application. Two promising directions for improving these models can be identified: (1) replacing the deterministic time horizon with a randomised one to achieve time-homogeneous backward dynamics, and (2) introducing conditioning in the forward process to reduce the need for extensive noise application.

We achieve the simultaneous implementation of both modifications in a unified framework employing the $h$-transform described above. To this end, we use an appropriate choice of the process $Z^h$ as our forward process and $\cev{Z^h}$ as backward process, as discussed next.

\subsection{Three possible implementations for unconditional sampling}
The examples from Section \ref{subsection:key_examples} immediately suggest possible realisations of our framework. Specifically, we use $Z^h$ as a forward process initialised in the data distribution $\alpha$ to learn the drift of the backward process $\cev{Z^h}$ (see the following sections for details). To generate (unconditional) samples from $\alpha$, we require suitable initial distributions for $\cev{Z^h}$, which are (approximately) given in the following examples:

\begin{enumerate}
    \item \textbf{Long exponential time horizons:}  
    If we choose an ergodic diffusion $Z$ (e.g., an Ornstein--Uhlenbeck process) and set $h=1$, we can pick $r$ small and let $\zeta$ follow an $\mathrm{Exp}(r)$ distribution. Under this setup, $\PR_\alpha(Z^h_{\zeta-} \in dy)$ is close to the stationary distribution of $Z$. Thus, we may initialize the backward process $\cev{Z^h}$ in the stationary distribution of $Z$ (if it exists).

    \item \textbf{Exponential bridge:}  
    In the second example, the forward process $Z^h$ is an exponential bridge targeting a specific state $x_1$. Here, we can sample from the data distribution by starting the backward process $\cev{Z^h}$ at $Z^h_{\zeta-}=x_1$.

    \item \textbf{Hitting a large sphere:}  
    In the third example, if $\alpha$ is compactly supported and the radius $R$ is chosen sufficiently large, $\PR_\alpha(Z^h_{\zeta-} \in dy)$ is approximately a uniform distribution over the sphere, which can then serve as an initial distribution for $\cev{Z^h}$. If helpful, the forward process can be approximated by \eqref{eq:OU-approx}, which allows a more direct sampling.
\end{enumerate}

In the idealised setting where $\alpha$ is known analytically,  generation via the backward process for unconditional sampling, i.e., sampling from\ $\alpha$, proceeds as described in Algorithm \ref{alg:gen_all_known}. This algorithm is derived from the results presented in Section \ref{subsec:time_reversal}.  
\begin{algorithm}[tb]
   \caption{Idealised generation when no learning is necessary}
   \label{alg:gen_all_known}
\begin{algorithmic}
   \STATE {\bfseries Input:}  $r> 0$, Green kernel $G_r$ of diffusion $Z$, forward transform $h$, target distribution $\alpha$, backward initial distribution $\beta^\prime \approx \PR_\alpha(Z^h_{\zeta-})$
    \STATE Set $\cev{h}(x) = \int \frac{G_r(x,y)}{h(y)}\, \alpha(dy)$
    \STATE {\bfseries Generation:}
    \STATE draw $x \sim \beta^\prime$
   \STATE Simulate path $(y_t)_{t \in [0,\zeta)}$ of $h$-transform  $Y = \cev{Z^h}$ initialised in $x$ until lifetime $\zeta$, with dynamics given in Proposition \ref{prop:time_inversion};
  \STATE {\bfseries Output:} $y_{\zeta-}$
\end{algorithmic}
\end{algorithm}
The simulation of the backward process up to the lifetime $\zeta$ can be implemented using standard numerical methods, such as the Euler--Maruyama scheme. In the most general setup, the backward process terminates at a randomised Markovian stopping time of the form
\[
\zeta = \inf\{t : A^h_t \geq E\},
\]
where $E$ is an independent exponential random variable and $A^h_t$ is an additive functional supported within the data distribution, which describes the measure of killing. For a detailed mathematical discussion and numerical approximations of $A^h$, see \cite{bally1989approximation,stoica1992approximation,christensen2024existence}. This framework directly enables simulation but poses challenges for estimation. However, the situation becomes significantly simpler when the data distribution $\alpha$ resides within a polar set, as will be discussed in the following section.

\subsection{Polarity hypothesis}
The target distribution $\alpha$ is typically unknown in practice. Consequently, we lack direct access to $\cev{h}$ as defined in \eqref{eq:back_h}, which is required to employ the backward generating process $Y = \cev{Z^h}$. Thus, $\cev{h}$ must be inferred from the data. As discussed in Section \ref{sec:theory}, both the drift of the backward process and the mechanism governing killing must be learned. While the killing mechanism can be described by a measure on the state space $\state$ (the killing measure), this adds significant complexity to the learning process.

A natural approach is to terminate the backward process as soon as a meaningful element of the data distribution is encountered. Proposition \ref{prop:polar} offers a criterion for when this is feasible: specifically, when the data distribution is concentrated in a polar set for the forward process, such as a lower-dimensional manifold. This aligns naturally with the manifold hypothesis: this well-explored concept in the literature assumes that high-dimensional data typically lie on or near a  low-dimensional manifold (or a union of such) \cite{loaiza24} and has been empirically verified for image data \cite{pope21, brown23}.

In our framework, we adopt a slightly more general assumption, referred to as the \emph{polarity hypothesis}, which posits that the data resides in a polar set. This generalisation allows us to disentangle learning of the drift and the killing mechanism, as detailed in the following sections.

\subsection{Learning the drift}
Here we focus on the case where the polarity hypothesis is satisfied, such  that by Proposition \ref{prop:polar} the lifetime of $\cev{Z^h}$ is given by the first hitting time of $\Omega \coloneq \operatorname{supp} \alpha$, which we assume to be known for the moment.
Based on Proposition \ref{prop:time_inversion}, we  aim to fit the data to a class of time-homogeneous diffusion processes 
\begin{align*}
dY^\theta_s &= \cev{b^\theta}(Y^\theta_s)\, ds + \sigma(Y^\theta_s) \, dW_s, \quad Y^\theta_0 \sim \PR_\alpha(Z^h_{\zeta-} \in \cdot),\\
\cev{b^\theta}(y) &= b(y) + \sigma(y)\sigma(y)^\top s^\theta(y),
\end{align*}
induced by a suitable class of candidate functions $\mathcal{S} = \{s^\theta: \theta \in \Theta\}$, and for an optimization objective and corresponding optimiser $\theta^\ast$ to be determined below, run the backward process $Y^{\theta^\ast}$ until an appropriate stopping time. Ideally, we would like to stop in the first hitting time of $\Omega$ by $Y^{\theta^\ast}$, i.e., in $\zeta^{\theta^\ast} = \tau_\Omega(Y^{\theta^\ast})$, denoting $\tau_A(X) \coloneqq \inf\{t: X_t \in A\}$ for a process $X$ and a set $A \subset \R^d$.  However, in the typical situation where $\Omega$ is a lower dimensional manifold with no simple structure, designing a candidate class $\mathcal{S}$ that guarantees $\PR(\zeta^{\theta^\ast} < \infty) = 1$ is a difficult and  highly problem-specific task (e.g., a $C^2$ submanifold $\Omega$ is polar for non-degenerate diffusions if its Hausdorff dimension is no  larger than $d-2$ \citep[Chapter 11]{friedman75}, \citep{rama88}). Instead, for some small $\varepsilon > 0$, we consider the closed $\varepsilon$-environment $\Omega_\varepsilon \coloneqq \{x \in \R^d: d(x,\Omega) \leq \varepsilon\}$ and run $Y^{\theta^\ast}$ until 
\[\zeta^{\theta^\ast}_\varepsilon = \tau_{\Omega_\varepsilon}(Y^{\theta^\ast}).\]
Thus, we are not targeting $\alpha$ directly, but the distribution  of $\cev{Z^h_{\zeta_\varepsilon}}$ for 
\[\zeta_\varepsilon \coloneqq \tau_{\Omega_\varepsilon}(\cev{Z^h}).\] 
This is justified for small $\varepsilon > 0$, since by continuity of the sample paths, we have that $\tau_{\Omega_\varepsilon}(\cev{Z^h})$ almost surely increases to $\zeta$ as $\varepsilon \downarrow 0$ if $\Omega$ is closed, and therefore obtain the  a.s.\ convergence 
\(\lim_{\varepsilon \downarrow 0} \cev{Z^h_{\zeta_\varepsilon}} = \cev{Z^h_{\zeta-}} \sim \alpha. \)
This is comparable to early stopping of the generating process in standard diffusion models with deterministic time horizon, which implies that the generative model targets the original data set blurred by some small Gaussian noise. On a forward time scale, the first entrance time in $\Omega^\varepsilon$ of the backward process $\cev{Z^h}$ corresponds to the last exit time $\sigma_\varepsilon$ of $\Omega_\varepsilon$ by the forward process $Z^h$, that is, 
\[\sigma_\varepsilon \coloneqq \sup\{t < \zeta: Z^h_t \in \Omega_\varepsilon\} = \zeta - \zeta_\varepsilon.\]
As a natural optimisation objective, we therefore target the Kullback--Leibler divergence between the law  $\cev{\mathbb{P}}{}^{\varepsilon}$ of the theoretical generating process $\cev{Z^h}$ killed in $\zeta_\varepsilon$ and the law $\mathbb{P}^{\theta,\varepsilon}$ of the parametrised generating process $Y^\theta$ killed at first entrance into $\Omega_\varepsilon$, i.e., at $\zeta^{\theta}_\varepsilon \coloneqq \tau_{\Omega_\varepsilon}(Y^{\theta})$. To this end, we set 
\begin{align*} 
\theta^\ast &\in \argmin_{\theta \in \mathcal{S}} \mathcal{L}_{\text{ex}}(\theta), \\ 
\mathcal{L}_{\text{ex}}(\theta) &\coloneqq \Ex\Big[\int_0^{\zeta_\varepsilon} \lVert  s^\theta(\cev{Z^h_t}) - \nabla \log \cev{h}(\cev{Z^h_t}) \rVert^2 \, dt  \Big]\\
&= \Ex\Big[\int_{\sigma_\varepsilon}^{\zeta} \lVert  s^\theta(Z^h_t) - \nabla \log \cev{h}(Z^h_t) \rVert^2 \, dt  \Big]
\end{align*}
which is the natural analog to the \textit{explicit score matching} objective in standard diffusion models.
Then, given sufficient integrability conditions, Girsanov's theorem indeed yields that 
\begin{align*}
 \mathrm{KL}\big(\cev{\PR}{}^{\varepsilon} \,\|\, \mathbb{P}^{\theta,\varepsilon}  \big)
 &= \frac{1}{2}\Ex\Big[\int_0^{\zeta_\varepsilon} \lVert \sigma^{-1}(\cev{Z^h_t})( \cev{b^h}(\cev{Z^h_t}) - \cev{b^\theta}(\cev{Z^h_t})) \rVert^2 \, dt  \Big]   \\ 
 &\leq \frac{C_\sigma}{2}\mathcal{L}_{\mathrm{ex}}(\theta),
\end{align*}
for $C_\sigma \coloneqq \sup_{x \in \R^d} \lVert \sigma^{-1}(x) \rVert^2$. Consequently,  
\[
\mathrm{KL}\big(\PR(\cev{Z^h}_{\hspace{-5pt}\zeta_\varepsilon} \in \cdot ) \,\|\, \PR(Y^{\theta^\ast}_{ \zeta^{\theta^\ast}_\varepsilon} \in \cdot ) \big)\lesssim  \min_{\theta \in \mathcal{S}} \Ex\Big[\int_{0}^{\zeta_\varepsilon} \lVert  s^\theta(\cev{Z^h_t}) - \nabla \log \cev{h}(\cev{Z^h_t}) \rVert^2 \, dt  \Big],
\]
showing that if $\mathcal{S}$ is rich enough to guarantee a high approximation quality of $\nabla \log \cev{h}$ on $(\Omega_\varepsilon)^{\mathrm{c}}$ (e.g., a suitable class of neural networks), our generative process produces high quality samples for slightly blurred target data. However, the general inaccessibility of $\cev{h}$, and thus of $\mathcal{L}_{\text{ex}}$, requires us to determine a tractable training objective that is comparable to $\mathcal{L}_{\text{ex}}$. This is provided by our next result, which in our proposed model can be understood as an analog of the \textit{denoising score matching loss} \citep{vincent11} employed in standard diffusion models.

\begin{proposition} \label{prop:score}
Given sufficient integrability conditions, there exists a constant $C$ independent of $\theta$, such that
\begin{align*} 
\mathcal{L}_{\mathrm{ex}}(\theta) = \mathcal{L}(\theta) - C_\varepsilon(\theta) + C,
\end{align*}
where
\[\mathcal{L}(\theta) \coloneqq \Ex\Big[\int_{\sigma_\varepsilon}^\zeta \lVert s^\theta(Z^h_t) - \nabla_{Z^h_t} \log G_r(Z^h_t, Z^h_0) \rVert^2 \, dt\Big] \] 
and 
\[C_\varepsilon(\theta) \coloneqq 2\Ex\Big[\int_0^{\sigma_{\varepsilon}} \big\langle s_\theta(Z^h_t), \nabla_{Z^h_t} \log \tfrac{G_r(Z^h_t,Z^h_0)}{ \cev{h}(Z^h_t)} \big\rangle \, d{t} \Big].\]
\end{proposition}
Since we assumed $\Omega$ to be polar, $C_\varepsilon(\theta)$ vanishes as $\varepsilon \to 0$, so that for small $\varepsilon > 0$ we expect minimisers of $\mathcal{L}$ -- which is accessible for a given Green kernel $G_r$ and can be efficiently approximated by a Monte-Carlo estimator $\hat{\mathcal{L}}$ based on simulated forward trajectories of $Z^h$  -- to be also approximate minimisers of $\mathcal{L}_{\text{ex}}$. This motivates the  generative Algorithm \ref{alg:gen_supp_known}, which outputs an estimated drift parameter $\hat{\theta}$ and a corresponding sample from the estimated backward generative process $Y^{\hat{\theta}}$ initialised in a distribution $\beta^\prime$ approximating the true forward terminal distribution $\PR_\alpha(Z^h_{\zeta-} \in \cdot)$. We note once again that a significant difference from most other methods is that there is no time dependence to be learned. This is also a feature of the first hitting model  from \cite{ye22}, which, however, relies on non-polarity of $\Omega$ and access to the $\Omega$-dependent Poisson kernel $\PR_x(Z_{\tau_\Omega} \in dz)$ for training and generation purposes. Usually, this is analytically tractable only for simple sets $\Omega$.  In contrast, our learning and generation step only requires access to the Green kernel $G_r$ of the \textit{unconditioned} forward process $Z$, which is independent of $\Omega$ and therefore allows for enhanced flexibility.
\begin{algorithm}[tb]
   \caption{Generation for unknown $\cev{h}$ and known polar data support $\Omega$}
   \label{alg:gen_supp_known}
\begin{algorithmic}
   \STATE {\bfseries Input:} data $\{y_i\}_{i=1}^n \overset{iid}{\sim} \alpha$,  $N \in \mathbb{N}$, $\varepsilon > 0, r> 0$, Green kernel $G_r$ of diffusion $Z$, enlarged data support $\Omega_\varepsilon$, forward transform $h$, backward initialisation $\beta^\prime \approx \PR_\alpha(Z^h_{\zeta-} \in \cdot)$ with $\operatorname{supp} \beta^\prime \subset \Omega_\varepsilon^{\mathrm{c}}$, function class $\mathcal{S}$.
   \FOR{$j=1$ {\bfseries to} $N$}
   \STATE draw $y_{i_j}$  with replacement from $\{y_i\}_{i=1}^n$ and simulate path $(z^{h,i_j}_t)_{t \in [0,\zeta^{i_j})}$ of $h$-transform $Z^h$ started in $y_{i_j}$;
   \ENDFOR
   \STATE {\bfseries Training:} Learn $\hat{\theta} = \argmin_{\theta \in \mathcal{S}} \hat{\mathcal{L}}(\theta)$ for 
   \[\hat{\mathcal{L}}(\theta) = \frac{1}{N} \sum_{j=1}^N \int_{\sigma^{i_j}_{\varepsilon}}^{\zeta^{i_j}} \lVert s^\theta(z^{h,i_j}_t) - \nabla_{1} \log G_r(z^{h,i_j}_t, y_{i_j}) \rVert^2 \, dt\]
   \STATE {\bfseries Generation:}  
   \STATE Draw $x$ from $\beta^\prime$
   \STATE Simulate path $(y^{\hat{\theta}}_t)_{t \in [0, \tau_{\Omega_\varepsilon}(y^{\hat{\theta}})]}$ of $Y^{\hat{\theta}}$ initialised in $x$
  \STATE {\bfseries Output:} $\hat{\theta}, y^{\hat{\theta}}_{\tau_{\Omega_\varepsilon}(y^{\hat{\theta}})}$
\end{algorithmic}
\end{algorithm}

\subsection{(Implicit) manifold learning}\label{sec:manifold}
We now turn to the most general setting, where the polar data support $\Omega = \operatorname{supp} \alpha$ is unknown. Since the generative mechanism proposed in Algorithm \ref{alg:gen_supp_known} fundamentally requires $\Omega_\varepsilon$ as input in order to know where to stop the generative process, we must infer $\Omega_\varepsilon$ from the data. To this end, we discuss two different strategies: first, a traditional statistical plug-in approach that simply replaces $\Omega_\varepsilon$ in Algorithm \ref{alg:gen_all_known} with a separately obtained estimate $\hat{\Omega_\varepsilon}$. The second approach is based on the observation that under the polarity hypothesis, the unknown drift component $\nabla \log \cev{h}$ implicitly encodes the geometry of the data manifold $\Omega$, cf.\ Proposition \ref{prop:stop}. This motivates an adaptive stopping criterion that only requires an estimate of $\nabla \log \cev{h}$ as input and may be interpreted as a first hitting time of an implicit manifold estimate.

\paragraph{Approach based on separation of manifold and drift estimation}
A first---but perhaps na\"{i}ve---strategy is to build an estimator $\widehat{\Omega_\varepsilon}$ based on the given data in a pre-processing step, and then use a plug-in approach to draw samples based on Algorithm \ref{alg:gen_supp_known}. This is formalised in Algorithm \ref{alg:gen_nothing_known}.
\begin{algorithm}
\caption{Full ``na\"{i}ve'' generative algorithm}
   \label{alg:gen_nothing_known}
\begin{algorithmic}
   \STATE {\bfseries Input:} data $\{y_i\}_{i=1}^n \overset{iid}{\sim} \alpha$,  $\varepsilon > 0, r> 0$, Green kernel $G_r$ of diffusion $Z$, estimator $\widehat{\Omega_\varepsilon}$ of enlarged data support, forward transform $h$, backward initial distribution $\beta^\prime \approx \PR_\alpha(Z^h_{\zeta-})$ with $\operatorname{supp} \beta^\prime \subset \widehat{\Omega_\varepsilon}^{\mathrm{c}}$, function class $\mathcal{S}$.
   \STATE Set $\Omega_\varepsilon \leftarrow \widehat{\Omega_\varepsilon}$
   \STATE {\bfseries Generation:} Run Algorithm \ref{alg:gen_supp_known}
  \STATE {\bfseries Output:} $\hat{\theta}, y^{\hat{\theta}}_{\tau_{\widehat{\Omega_\varepsilon}}(y^{\hat{\theta}})}$
\end{algorithmic}
\end{algorithm}

To construct $\widehat{\Omega_\varepsilon}$ given a data sample $(Y_i)_{i=1}^n \overset{iid}{\sim} \alpha$, we can consider the following two natural approaches: 

First, we may target $\Omega_\varepsilon$ directly by  blurring the data via $Y_i^\varepsilon = Y_i + \eta^\varepsilon_i$, where $(\eta^\varepsilon_i)_{i=1}^n$ is some iid noise that is independent of the data sample and absolutely continuous with support $B(0,\varepsilon)$, e.g., $\eta^\varepsilon_i \sim \mathcal{U}(B(0,\varepsilon))$. Then, $Y_i^\varepsilon$ is absolutely continuous with  support $\Omega_\varepsilon$ and density $\pi^\varepsilon(x) =\int_\Omega \pi_{\eta^\varepsilon}(x-y) \, \alpha(dy)$. Recovering the compact support $\Omega_\varepsilon$ of such a data sample is a well-studied statistical problem. The perhaps most common approaches are plug-in estimators \cite{cuevas97} that set \[\widehat{\Omega_\varepsilon} = \{\widehat{\pi^\varepsilon} > \delta_n\}\] 
for some nonparametric estimator $\widehat{\pi^\varepsilon}$ of $\pi^\varepsilon$ and a tuning parameter $\delta_n \to 0$, or the simple and intuitive \textit{Devroye--Wise} estimator \cite{devroye80} given by 
\[\widehat{\Omega_\varepsilon} = \bigcup_{i=1}^n B(Y_i^\varepsilon, \delta_n).\]
Minimax optimality of these estimators with respect to metrics such as the Hausdorff metric or the symmetric difference volume has been established under different assumptions on the data density $\pi^\varepsilon$ and the geometry of $\Omega_\varepsilon$ \cite{korestelev93, mammen95, cuevas97, cuevas04, biau08}. The rates, however, generally slow down exponentially in terms of the ambient data dimension $d$, which motivates the second approach that allows to exploit directly the lower-dimensional structure of $\Omega$.

For this second approach, instead of estimating $\Omega_\varepsilon$ directly, we may first construct an estimator $\widehat{\Omega}$ of the true data support $\Omega$ based on the unmodified data sample $(Y_i)_{i=1}^n$ and then set 
\[\widehat{\Omega_\varepsilon} = (\widehat{\Omega})_\varepsilon.\]
This approach might be more promising for high-dimensional data since a lot of progress has been made in recent years on statistical theory and algorithmic implementation for support estimation of distributions concentrated on lower-dimensional manifolds. Important contributions that provide estimators $\widehat{\Omega}$ with provable convergence rates only depending on the smoothness $s$ and the dimension $d^\prime < d$ of the data manifold include \cite{genovese12,ma12,kim15,aamari18,aamari19,aamari23}. While computationally  more involved than the simple estimators based on noised data discussed above, the estimators from \cite{aamari18,aamari19,aamari23} are constructive and implementable.

\paragraph{Fully integrated drift and support estimation}
The plug-in approach is statistically sound, but it is rather unclear if it will lead to convincing results in actual implementations. Indeed, part of the empirical success of standard diffusion models is explained by an implicit adaptation to the data support via the learned score networks \cite{stan24}. In this spirit, we propose a third strategy that only relies on information obtainable from the approximation of the log-gradient of the reverse $h$-transform. This is based on the intuition that $\nabla \log \cev{h}(x)$ must explode when $x$ approaches the low-dimensional data support $\Omega$ to ensure that the reverse diffusion process is killed in finite time. To formalise this intuition in the following, let us first give a  useful characterisation of  $\nabla \log \cev{h}$. The proof is given in Appendix \ref{app:score}.

\begin{lemma}\label{lem:char_score}
For all $x \notin \supp \alpha$, it holds that 
\begin{equation}\label{eq:score_char1}
\nabla \log \cev{h}(x) = \E\big[\nabla_x \log G_r(x, Z^{\cev{h}}_{\zeta-}) \mid Z^{\cev{h}}_0 = x \big].
\end{equation}
In particular, for $\PP_\alpha(Z^h_{\zeta_-} \in\cdot)$-a.e.\ $x$,  it holds that
\begin{equation} \label{eq:score_char2}
\nabla \log \cev{h}(x) = \E_\alpha\big[\nabla_x \log G_r(Z^{h}_0,x) \mid Z^{h}_{\zeta_-} = x \big].
\end{equation}
\end{lemma}
\begin{remark} 
\begin{enumerate} 
\item[(i)] If the forward process is simply an exponentially killed diffusion (i.e., $h=1$), the representation of $\nabla \log \cev{h}$ as a conditional expectation can be used to motivate a simplified denoising score matching objective that only requires simulation of the forward process at its exponential lifetime, cf.\ \cite[Theorem 2.7]{christensen25} for the case of killed Brownian motion.\label{rem:score}  
\item[(i)]
Formula \eqref{eq:score_char2} is analogue to the well-known score representation 
\[\nabla \log p_t(x) = \E_\alpha\big[\nabla_x \log p_t(\vec{X}_0,x) \mid \vec{X}_t = x\big]\]
in conventional diffusion models for a homogeneous forward process $\vec{X}$ with transition densities $p_t(x,y)$, which is linked to Tweedie's formula \cite{robbins56}.
\end{enumerate}
\end{remark}
Let us for the moment consider the Green kernel $G_r$ of a Brownian motion. Then, we have 
\[\nabla_x \log G_r(x,y) = -\sqrt{2r}\operatorname{sign}(x-y)  \frac{K_{\nu+1}(\sqrt{2r}\lVert x-y\rVert)}{K_{\nu}(\sqrt{2r}\lVert x-y\rVert)}, \quad \nu = \frac{d-2}{2}, \,\operatorname{sign}(x) \coloneqq \frac{x}{\lVert x \rVert},\]
and therefore from above, for any $x \notin \supp \alpha$,
\begin{equation}\label{eq:score_brownian_asymp}
\nabla \log \cev{h}(x) = \sqrt{2r}\E^x\Bigg[\operatorname{sign}\big( Z^{\cev{h}}_{\zeta-} - x\big) \frac{K_{\nu+1}(\sqrt{2r}\lVert Z^{\cev{h}}_{\zeta-} -x\rVert)}{K_{\nu}(\sqrt{2r}\lVert Z^{\cev{h}}_{\zeta-}-x\rVert)} \Bigg] \approx \frac{d}{\mathrm{e}} \E_x\Bigg[\frac{\operatorname{sign}( Z^{\cev{h}}_{\zeta-}-x)}{\big\lVert  Z^{\cev{h}}_{\zeta-} -x \big\rVert}\Bigg],
\end{equation}
for large $d$, cf.\ 10.41 in \cite{NIST:DLMF}. A more careful analysis furthermore reveals that the scaling factor $d$ in \eqref{eq:score_brownian_asymp} can be replaced by $d - \dim \supp \alpha$ under mild regularity assumptions on the data manifold. This demonstrates quite clearly how the reverse process is dragged onto the data manifold. To exploit this connection for our purposes, we can use the explosive behavior of $\nabla \log \cev{h}$ near $\supp \alpha$ to obtain an insightful characterisation of the backward killing time, as the following result shows.

\begin{proposition}\label{prop:stop}
Suppose that 
\begin{enumerate}
\item[(i)] the drift $b$ of $Z$ is locally bounded and its diffusion coefficient $\sigma\sigma^\top$ is uniformly elliptic;
\item[(ii)] $\supp \alpha$ is polar for both $Z^h$ and $Y$ solving $\sigma(Y_t)\, \diff{B_t}$.
\end{enumerate}
Then, the  killing time $\zeta$ of $\cev{Z^h}$ is a.s.\ given by
\[\zeta = \inf\big\{t \geq 0: \sup_{s \leq t} \big\lVert \nabla \log \cev{h}(\cev{Z^{h}_s}) \big\rVert = \infty\big\} = \inf\big\{t \geq 0: \big\lVert \nabla \log \cev{h}(\cev{Z^{h}}) \big\rVert_{L^2([0,t])} = \infty\big\}.\]
\end{proposition}
\begin{proof}
The proof proceeds along the same lines as the one of Theorem 2.8 in  \cite{christensen25}, where the result is proved for the particular case of the forward process $Z^h$ being given by an exponentially killed Brownian motion.
\end{proof}

Most importantly in our context, the above considerations lead to simple and interpretable stopping criteria for the reverse process based on an estimate  $s^{\hat{\theta}}$ of $\nabla \log \cev{h}(x)$. For instance, when $Z$ is a Brownian motion and we target killing on the slightly enlarged data support $\Omega_\varepsilon$, then  $Z^{\cev{h}}_{\zeta-} \in \Omega$ implies for any $x \notin \Omega_\varepsilon$ that
\[\lVert \nabla \log \cev{h}(x) \rVert \approx \frac{d}{\mathrm{e}}\Bigg\lVert \E_x\Bigg[\frac{\operatorname{sign}( Z^{\cev{h}}_{\zeta-}-x)}{\big\lVert  Z^{\cev{h}}_{\zeta-} -x \big\rVert}\Bigg]\Bigg\rVert \leq \frac{d}{\mathrm{e}}\E_x\bigg[ \frac{1}{\lVert Z^{\cev{h}}_{\zeta-} - x\rVert}\bigg] \leq \frac{d}{\mathrm{e} \varepsilon}. \] 
Combining this with Proposition \ref{prop:stop}, we may therefore reasonly stop the data-generating mechanism as soon as  $\lVert s^{\hat{\theta}}(Y^{\hat{\theta}}_t) \rVert \gtrsim \frac{d}{\varepsilon}$. This translates into a score-based reverse first killing time 
\[\hat{\zeta_\varepsilon} \coloneqq \inf\big\{t \geq 0: \lVert s^{\hat{\theta}}(Y^{\hat{\theta}}_t) \rVert \geq \gamma d \varepsilon^{-1} \big\},\]
for some hyperparameter $\gamma$ or, equivalently, the first hitting time of the implicit data support estimator 
\[\hat{\Omega_\varepsilon} \coloneqq \big\{x \in \R^d: \lVert s^{\hat{\theta}}(x) \rVert \geq \gamma d \varepsilon^{-1}\big\}.\]
Such a stopping rule encodes mathematically, what a human observer would do intuitively when tracking the sequentially generated data: 

\begin{quote}
\textit{stop the generation process as soon the observed ordered set of pixels visually resembles a high quality image.}
\end{quote}

To obtain an estimator $s^{\hat{\theta}}$ based on Proposition \ref{prop:score} without prior estimation of $\Omega_\varepsilon$, we  simply replace the lower time index $\sigma_\varepsilon \coloneqq \sup\{t< \zeta: Z^h_t \in \Omega_\varepsilon\}$ in the training objective with $\zeta \wedge \underline{T}$ for some hard-coded small value $\underline{T} \geq 0$. For specific forward processes such as exponentially killed diffusions, more elaborate alternatives are theoretically justified, see \cite[Theorem 2.7]{christensen25} and Remark \ref{rem:score}. The full generative algorithm is given in Algorithm \ref{alg:integrated}. 

\begin{algorithm}[tb]
   \caption{Full generative algorithm without separate support estimation}
   \label{alg:integrated}
\begin{algorithmic}
   \STATE {\bfseries Input:} data $\{y_i\}_{i=1}^n \overset{iid}{\sim} \alpha$,  $N \in \mathbb{N}$, $\varepsilon > 0, \underline{T} \geq 0, \gamma > 0, r> 0$, (approximation of) Green kernel $G_r$ of diffusion $Z$, enlarged data support $\Omega_\varepsilon$, forward transform $h$, backward initialization $\beta^\prime \approx \PR_\alpha(Z^h_{\zeta-} \in \cdot)$, function class $\mathcal{S}$.
   \FOR{$j=1$ {\bfseries to} $N$}
   \STATE Draw $y_{i_j}$ uniformly with replacement from $\{y_i\}_{i=1}^n$ and simulate path $(z^{h,i_j}_t)_{t \in [0,\zeta^{i_j})}$ of $h$-transform $Z^h$ started in $y_{i_j}$;
   \ENDFOR
   \STATE {\bfseries Training:} Learn $\hat{\theta} = \argmin_{\theta \in \mathcal{S}} \hat{\mathcal{L}}(\theta)$ for 
   \[\hat{\mathcal{L}}(\theta) = \frac{1}{N} \sum_{j=1}^N \int_{\underline{T} \wedge \zeta^{i_j}}^{\zeta^{i_j}} \big\lVert s^\theta(z^{h,i_j}_t) - \nabla_{1} \log G_r(z^{h,i_j}_t, y_{i_j}) \big\rVert^2 \, dt\]
   \STATE {\bfseries Generation:}  
   \STATE Draw $x$ from $\beta^\prime$
   \STATE Simulate path $(y^{\hat{\theta}}_t)_{t \in [0, \hat{\zeta_\varepsilon}]}$ of $Y^{\hat{\theta}}$ initialised in $x$ until $\hat{\zeta_\varepsilon} = \inf\{t \geq 0: \lVert s^{\hat{\theta}}(y^{\hat{\theta}}_t) \rVert \geq \gamma d \varepsilon^{-1}\}$ 
  \STATE {\bfseries Output:} $\hat{\theta}, y^{\hat{\theta}}_{\hat{\zeta}}$
\end{algorithmic}
\end{algorithm}

\subsection{Noise-Schedules as a special case of radial $h$-transforms}
In classical diffusion models, the forward {noise schedule} $\alpha_t$ is of key importance for sampling. In our time-homogeneous setting this effect can be reproduced in high dimension by choosing a suitable $h$-transform for the forward process, exemplified here for a Brownian motion $Z$.

Consider the standard forward noising
\[
x_t=\sqrt{\alpha_t}\,x_0+\sqrt{1-\alpha_t}\,z,\qquad z\sim\mathcal N(0,I_d),\ \|x_0\|=o(\sqrt d),
\]
for which the law of large numbers gives $\|x_t\|/\sqrt d\longrightarrow \rho^*(t)\coloneq \sqrt{1-\alpha_t}$ as $d\longrightarrow\infty$. The limiting radial trajectory $\rho^*(t)$ satisfies
\begin{equation}\label{eq:ODE_inhomo}
\dot \rho^*(t)= -\,\frac{\dot\alpha_t}{2\sqrt{1-\alpha_t}}.
\end{equation}
To reproduce this within our framework, define $\Psi$ by
\[
\Psi'(\rho)=\tfrac{1}{2\rho}-\dot\rho^*((\rho^*)^{-1}(\rho)),\qquad \Psi(\rho_0)=0,
\]
on a radial range corresponding to $t\in[\tau,T]\subset(0,\infty)$, where $\rho^*$ is smooth and strictly monotone. 
If we can choose a positive radial $r$-potential $h$ whose high-dimensional log-profile satisfies
\begin{equation*}
    -\frac1d\log h(x)\approx \Psi(\rho),\qquad \rho=\|x\|/\sqrt d,
\end{equation*}
and assuming these log-asymptotics is differentiable on the considered radial range, then
\[
\nabla\log h(x)\approx -\sqrt d\,\Psi'(\rho)\,\frac{x}{\|x\|}.
\]
For $\rho_t=\|Z^h_t\|/\sqrt d$ we then have the approximation
\[
d\rho_t\approx \Big(\frac{d-1}{2d\,\rho_t} - \Psi'(\rho_t)\Big)\,dt+\tfrac{1}{\sqrt d}\,dB_t
\]
where $B$ is a one-dimensional Brownian motion.
As $d\longrightarrow\infty$, we obtain the autonomous ODE
\begin{equation}\label{eq:ODE_homo}
    \dot\rho=\frac{1}{2\rho}-\Psi'(\rho) = \dot\rho^*((\rho^*)^{-1}(\rho)).
\end{equation}
Hence, if we set $\rho(t)=\rho^*(t)$, then $(\rho^*)^{-1}(\rho(t))=t$ (since $\rho^*$ is strictly monotone) and the autonomous ODE \eqref{eq:ODE_homo} reduces to $\dot\rho(t)=\dot\rho^*(t)$, i.e.\ \eqref{eq:ODE_homo} and \eqref{eq:ODE_inhomo} describe the same radial trajectory (possibly up to a time shift determined by the initial condition).

Thus the two descriptions differ in form: the classical schedule induces an explicit time-dependent drift for $\rho^*$, while the $h$-transform encodes the same effect through a space-dependent drift $-\Psi'(\rho)$. By the choice of $\Psi$, both ODEs \eqref{eq:ODE_homo} and \eqref{eq:ODE_inhomo} coincide along the trajectory $\rho^*(t)$ in the high-dimensional limit; consequently, in high dimensions, a single radial $h$ reproduces the schedule-driven radial behavior in a time-homogeneous formulation.

\section{Features of the model}\label{subsec:features}
\subsection{Natural conditioning}\label{subsec:condition}
In the procedure described above, we have explained how to sample from the data distribution $\alpha$ using the time-homogeneous backward processes $Y = \cev{Z^h}$ or a learned version thereof when initialised from the terminal distribution of the forward process. The flexibility of our proposed model compared to existing diffusion models, however, arises from the observation that due to its time-homogeneous nature, the backward process $Y$ can be started in an interpretable way directly from \textit{any} initial state $x$, regardless of the specific time, along the following lines:
\begin{enumerate}
    \item Initialize the backward process $Y$ at a chosen state $x$, which could be:
    \begin{itemize}
        \item A noisy version of a sample from $\alpha$, or
        \item Any point of interest in the state space.
    \end{itemize}
    \item Simulate the backward process $Y$ until the lifetime $\zeta$, generating the trajectory $\{Y_t\}_{t \leq \zeta}$.
    \item Extract the value $Y_{\zeta-}$ as the final sample. Note that this value lies within the support of $\alpha$, but is not necessarily distributed exactly according to $\alpha$.
    \item Optionally, repeat the process for multiple initial states $x$ to analyze how the sampling distribution varies with the initialisation.
\end{enumerate}

In Bayesian terminology, unconditional sampling corresponds to drawing from the prior distribution $\alpha$, whereas the current task involves sampling from $\mathbb{P}_x(Y_{\zeta-} \in dy)$, the posterior distribution conditioned on the input data $x$.

It is reasonable to expect that $Y_{\zeta-}$ will typically be close to $x$. Indeed, Proposition~\ref{prop:time_inversion} states that
\[
\mathbb{P}_x(Y_{\zeta-} \in dy) = \frac{1}{\cev{h}(x)} \frac{G_r(x, y)}{h(y)}\, \alpha(dy).
\]
If $h(y)$ is approximately constant over the support of $\alpha$, the sampling procedure approximates drawing from the measure proportional to $G_r(x, y)\, \alpha(dy)$ (up to a normalizing constant). Typically, $G_r(x, y)$ decreases with the distance between $x$ and $y$, as illustrated in Remark~\ref{ex:G_r_BM}. Consequently, points $y$ closer to $x$ appear with higher frequency than those farther away.

In other words, if the backward process $Y$ is started near the support of $\alpha$, the resulting sample from $\alpha$ will tend to be close to $x$. Here, the choice of the diffusion process $Z$ implicitly defines a corresponding notion of distance. 
So by choosing the starting point of $Y$, we get natural conditioning, and in exactly the same (learned) model as for unconditional sampling.

\subsection{Fine tuning for conditional data generation}\label{sec:fine_tuning}
In addition to the natural conditioning properties discussed above, we may also use our model to explicitly sample from a conditional distribution of interest, which is only represented through a smaller subset of samples having certain properties. More specifically, suppose that instead of sampling from $\alpha$ we would like to sample from the posterior $\alpha(dy \mid u) \propto \pi(u \mid y)\, \alpha(dy)$, where $\pi(u)$ is the density of some latent random variable $U$ of interest and assume that the likelihood satisfies $\pi(u \mid \cdot) \in L^2(\alpha)$. To do so, we aim at lifting a generative model from our unconditional pre-trained model (i.e., using the learned score $\nabla \log \cev{h}$), for which more data is available. Ideally, we would like to simulate a backward process with drift determined by $\nabla \log \cev{h}(\cdot \mid u)$, where $\cev{h}(x \mid u) \coloneqq \int \tfrac{G_r(x,y)}{h(y)} \, \alpha(dy \mid u)$. This can be calculated similarly to Lemma \ref{lem:char_score} as follows 
\begin{equation} \label{eq:score_cond}
\begin{split}
\nabla \log \cev{h}(x \mid u)  &= \frac{\int \nabla_x G_r(x,y) \frac{1}{h(y)} \pi(u\mid y) \,\alpha(dy)}{\int G_r(x,y) \frac{1}{h(y)} \pi(u \mid y) \,\alpha(dy)}\\ 
&= \frac{\int \nabla_x \log G_r(x,y)  \pi(u\mid y) \frac{G_r(x,y)}{\cev{h}(x)h(y)} \,\alpha(dy)}{\int \pi(u \mid y) \frac{G_r(x,y)}{\cev{h}(x)h(y)} \, \alpha(dy)} \\ 
&= \frac{\E_x\big[\nabla_x \log G_r(x, Z^{\cev{h}}_{\zeta-}) \pi(u \mid Z^{\cev{h}}_{\zeta-}) \big]}{\E_x\big[\pi(u \mid Z^{\cev{h}}_{\zeta-}) \big]}\\ 
&= \E^u_x\big[\nabla_x \log G_r(x, Z^{\cev{h}}_{\zeta-})\big],
\end{split}
\end{equation}
where $\PP_x = \PP(\,\cdot \mid Z^{\cev{h}}_0 = x)$ and $\PP^u_x$ is  a probability measure defined by the Radon--Nikodym derivative 
\[\frac{d \PP^u_x}{d\PP_x} \coloneq \frac{\pi\big(u \mid Z^{\cev{h}}_{\zeta-}\big)}{\E_x\big[\pi\big(u \mid Z^{\cev{h}}_{\zeta-}\big)\big]},\]
relative to $\PP_x$. Recalling that Lemma \ref{lem:char_score} shows $\nabla \log h(x) = \E_x\big[\nabla_x \log G_r(x, Z^{\cev{h}}_{\zeta-})\big]$, we see that the conditional log gradient of the reverse $h$-transform is  obtained via a change of measure induced by the likelihood $\pi(u \mid y)$ and we obtain the decomposition 
\begin{align*}
\nabla \log \cev{h}(x \mid u) &=  \nabla \log \cev{h}(x) + \E_x\Big[\nabla_x\log G_r(x, Z^{\cev{h}}_{\zeta-})\tfrac{\pi(u \mid Z^{\cev{h}}_{\zeta-}) - \E_x[\pi(u \mid Z^{\cev{h}}_{\zeta-})]}{\E_x[\pi(u \mid Z^{\cev{h}}_{\zeta-})]} \Big].
\end{align*}
Furthermore,  \eqref{eq:score_cond} yields 
\begin{align*} 
\nabla \log \cev{h}(x \mid u) &= \frac{\E_{Z^h_0 \sim \alpha}\big[\nabla_x \log G_r(x, Z^h_0) \pi(u \mid Z^h_0) \, \big\vert\, Z^h_{\zeta_-} = x\big]}{\E_{Z^h_0 \sim \alpha}\big[\pi(u \mid Z^h_0) \, \big\vert\, Z^h_{\zeta-} =x\big]} \\ 
&=\int \nabla_x \log G_r(x,y) \, \PP_{Z^h_0 \sim \alpha(dy \mid u)}\big(Z^h_{0} \in dy \,\big\vert\, Z^h_{\zeta-} =x)\\ 
&= \E_{Z^h_0 \sim \alpha(dy \mid u)}\big[\nabla_{Z^h_{\zeta-}} \log G_r(Z^h_{\zeta_-}, Z^h_0 )\,\big\vert\, Z^h_{\zeta_-} = x\big],
\end{align*}
where we used that  by Bayes' theorem, 
\[\pi(u \mid y) \PP_{Z^h_0 \sim \alpha}(Z^h_{0} \in dy \mid Z^h_{\zeta-} = x) \propto \PP_{Z^h_0 \sim \alpha(dy \mid u)}(Z^h_0 \in dy \mid Z^h_{\zeta-} = x).\]
This may be interpreted as an analogue of the score matching identity for the posterior score in diffusion models. Similarly to the fine-tuning approach from \cite{denker24}, given a pre-processed estimator $s^{\hat{\theta}}$ of $\nabla \log \cev{h}$ and a smaller subset of training data for the conditional distribution $\alpha(dy \mid u)$, we may now aim at approximating only the difference 
\[D_u(x) \coloneq \nabla \log \cev{h}(x \mid u) - \nabla \log \cev{h}(x) = \E\Big[\nabla_x\log G_r(x, Z^{\cev{h}}_{\zeta-})\tfrac{\pi(u \mid Z^{\cev{h}}_{\zeta-}) - \E[\pi(u \mid Z^{\cev{h}}_{\zeta-}) \mid Z^{\cev{h}}_0 = x]}{\E[\pi(u \mid Z^{\cev{h}}_{\zeta-}) \mid Z^{\cev{h}}_0 = x]} \,\Big\vert\, Z^{\cev{h}}_0 = x \Big],\]
of the log-gradient of the conditional and unconditional $h$-transform. From \cite[Lemma A.1]{christensen25}, it follows  that for any $\delta > 0$, $D_u$ coincides on $(\Omega_\delta)^{\mathrm{c}}$ with the minimiser of 
\[D \mapsto \E_{Z_0^h \sim \alpha(dy \mid u)}\Big[\big\lVert  D(Z^h_{\zeta-}) + \nabla \log \cev{h}(Z^h_{\zeta-}) - \nabla_{Z^h_{\zeta -}}\log G_r(Z^h_0,Z^h_{\zeta-}) \big\rVert^2 \one_{(\delta,\infty)}(\lVert Z^h_{\zeta-} - Z^h_{0} \rVert) \Big], \]
in the space of measurable functions, where conditioning on $\{ \lVert Z^h_{\zeta-} - Z^h_{0} \rVert > \delta\}$ ensures that the expectation is finite. Given $s^{\hat{\theta}} \approx \nabla \log {\cev{h}}$, this may be implemented as an empirical version of
\[\argmin_{D \in \mathcal{D}} \E_{Z_0^h \sim \alpha(dy \mid u)}\Big[\big\lVert  D(Z^h_{\zeta-}) + s^{\hat{\theta}}(Z^h_{\zeta-}) - \nabla_{Z^h_{\zeta -}}\log G_r(Z^h_0,Z^h_{\zeta-}) \big\rVert^2 \one_{(\delta,\infty)}(\lVert Z^h_{\zeta-} - Z^h_{0} \rVert) \Big],\]
for some  class of approximating functions $\mathcal{D}$. Given such minimiser $\hat{D}_u$, we then define a conditional score estimator as 
\[\hat{s}(y \mid u) \coloneq \hat{D}_u(y) + s^{\hat{\theta}}(y). \]
The corresponding conditional generation procedure is summarised in Algorithm \ref{alg:cond}.
\begin{algorithm}[tb]
   \caption{Conditional generation via fine tuning of unconditional algorithm}
   \label{alg:cond}
\begin{algorithmic}
   \STATE {\bfseries Input:} data $\{y_i\}_{i=1}^n \overset{iid}{\sim} \alpha$,  $M \in \mathbb{N}$, $\varepsilon,\delta,\gamma > 0, r> 0$, Green kernel $G_r$ of diffusion $Z$,  forward transform $h$, backward initialisation $\beta^\prime_u \approx \PR_{\alpha(dy \mid u)}(Z^h_{\zeta-} \in \cdot)$, function class $\mathcal{D}$, unconditional log-gradient $h$-transform estimator $\hat{s} = s^{\hat{\theta}} \approx \nabla \log {\cev{h}}$.
   \FOR{$j=1$ {\bfseries to} $M$}
   \STATE Draw $y_{i_j}$ uniformly from $\{y_1,\ldots,y_n\}$ and simulate terminal value $z^{h,i_j}_{\zeta^{i_j}_-}$ of $h$-transform $Z^h$ started in $y_{i_j} \sim \alpha(dy \mid u)$;
   \ENDFOR
   \STATE {\bfseries Training:} Learn $\hat{D}_u \in \argmin_{D \in \mathcal{D}} \hat{\mathcal{L}}_u(D)$ for 
   \[\hat{\mathcal{L}}_u(D) = \frac{1}{M} \sum_{j=1}^M \lVert D(z^{h,i_j}_{\zeta^{i_j}-}) + \hat{s}(z^{h,i_j}_{\zeta^{i_j}_-}) - \nabla_{z^{h,i_j}_{\zeta-}} \log G_r(z^{h,i_j}_{\zeta^{i_j}-}, y_{i_j}) \rVert^2 \one_{(\delta,\infty)}(\lVert z^{h,i_j}_{\zeta^{i_j}-} - y_{i_j} \rVert)\]
   \STATE set $\hat{s}(\cdot \mid u)  = \hat{D}_u + s^{\hat{\theta}}$
   \STATE {\bfseries Generation:}  
   \STATE Draw $x$ from $\beta^\prime_u$
   \STATE Simulate path $(\hat{y}^u_t)_{t \in [0, \hat{\zeta}_{\varepsilon}]}$ of 
   \[dY^u_t = (b(Y^u_t) + \sigma\sigma^\top(Y^u_t)\hat{s}(Y^u_t \mid u)) \,dt + \sigma(Y^u_t) \, dW_t, \quad Y^u_0 =x,\]
   until $\hat{\zeta}_\varepsilon = \inf\{t \geq 0: \lVert \hat{s}(\hat{Y}^u_t \mid u) \geq \gamma d\varepsilon^{-1}\}$
  \STATE {\bfseries Output:} $\hat{y}^u_{\hat{\zeta}_\varepsilon}$
\end{algorithmic}
\end{algorithm}

As an alternative to this fine-tuning approach, we may  use the probabilistic characterisation \eqref{eq:score_cond} directly, provided that we have access to the likelihood $\pi(u \mid y)$ or a good approximation thereof (e.g., a classifier). To this end, we may first approximate the expected likelihood function
\[x \mapsto \E_x[\pi(u \mid Z^{\cev{h}}_{\zeta_-})] = \E_{Z^h_0 \sim \alpha}[\pi(u \mid Z^h_0) \mid Z^h_{\zeta_-} = x]\] 
which minimises the $L^2$-loss 
\[\ell_u \mapsto \E_{Z^h_0 \sim \alpha}\big[\lvert \ell_u(Z^h_{\zeta-}) - \pi(u \mid Z^h_0) \rvert^2\big],\] 
in $L^2(\alpha)$.
Denote the solution of the associated nonparametric regression problem (based on a subset $\{y_1,\ldots,y_{n_1}\}$ of the unconditional data set $\{y_1,\ldots, y_n\}$ and some class of candidate likelihood functions $\mathcal{P}_u$) by $\hat{\ell}_u$, and similarly, using the same reasoning as above  based on \cite[Lemma A.1]{christensen25}, estimate $\nabla \log \cev{h}(\cdot \mid u)$ by a  minimiser of
\[s_u \mapsto \frac{1}{M} \sum_{j=1}^M \Big\lVert s_u\big(z^{h,i_j}_{\zeta^{i_j}-}\big) - \nabla_2 G_r(y_{i_j}, z^{h,i_j}_{\zeta^{i_j}-}) \frac{\pi(u \mid y_{i_j})}{\hat{\ell}_u(z^{h,i_j}_{\zeta^{i_j}-}) \vee \varepsilon} \Big\rVert^2 \one_{\big\{\lVert z^{h,i_j}_{\zeta^{i_j}-} - y_{i_j} \rVert > \delta\big\}}\]
within some approximating function class $\mathcal{S}_u$. Here, $\delta,\varepsilon$ are regularisation parameters, the indices $i_j$ are uniformly drawn with replacement from the remaining data indices $\{n_1+1,\ldots,n\}$ and $z^{h,i_j}_{\zeta^{i_j}-}$ is a simulated value of the $h$-transform $Z^h$ at its lifetime when initialised in the data point $y_{i_j}$.
Note here that the estimation procedure neither requires an estimator of $\nabla \log \cev{h}$ nor explicit access to samples of $\alpha(dy \mid u)$. The separate estimation of the denominator $\ell_u$ may however cause numerical challenges. These are more pronounced if  $\pi(u \mid Z^h_0)$ is concentrated around $0$ under $\PP_{Z^h_0 \sim \alpha}$ corresponding to scarce data availability for the condition $U = u$.

\subsection{Distances to the distribution and anomaly detection}
We now address the question of how to measure the distance of an input $ x $ from the data distribution $ \alpha $. The previous discussion provides a natural measure for this, namely the time it takes the forward process $Y$ to transform the input $x$ into a sample from $\alpha$.

We can identify $x$ as an anomaly if the average lifetime exceeds a certain threshold $\underline{T}$ using the Monte-Carlo procedure given in Algorithm \ref{alg:anomaly}.

\begin{algorithm}
\caption{Anomaly detection}
   \label{alg:anomaly}
\begin{algorithmic}
   \STATE {\bfseries Input:} state $x$ to be evaluated,  
   learned score $s^{\hat{\theta}}$ based on Algorithm \ref{alg:integrated}, threshold $\underline{T} > 0$, \# of Monte-Carlo runs $N$, precision parameters $\gamma,\varepsilon > 0$
   \STATE Initialize $\overline{\zeta} = 0$
   \IF{$\lVert s^{\hat{\theta}}(x) \rVert < \gamma d \varepsilon^{-1}$ }
    \FOR{$i=1$ {\bfseries to} $N$}
        \STATE Initialize $Y^{\hat{\theta}}_0 = x$
        \REPEAT
            \STATE Simulate next step  of $Y^{\hat{\theta}}$ 
        \UNTIL{$\lVert s^{\hat{\theta}}(Y^{\hat{\theta}}_t) \rVert \geq \gamma d \varepsilon^{-1}$}
        \STATE Update $\overline{\zeta} \leftarrow \overline{\zeta} + t/N$
    \ENDFOR
   \ENDIF
   \IF{$\overline{\zeta} > \underline{T}$}
    \STATE Classify as anomaly
   \ENDIF
\end{algorithmic}
\end{algorithm}

This can also serve as a criterion for determining whether new training data necessitates a modification of our trained model or not.

\subsection{Class sampling and classification}
The model can be naturally extended when the data can be decomposed into subclasses, i.e., when $\alpha$ is a mixture of distributions $\alpha_1, \ldots, \alpha_n$ with disjoint supports $\Omega^i$. In this scenario, there is a corresponding decomposition $\cev{h} = \sum_{i=1}^n \cev{h}_i$. The lifetime $\zeta$ is then determined as the minimum of the first entry times $\zeta_i$ into the supports of the $\alpha_i$. With unconditional sampling, this provides information about the specific class from which the sampled image originates. 

The truly interesting aspect of class decomposition becomes evident when we consider the conditional setting as the previous observations induce the  natural classification Algorithm \ref{alg:classification}  based on the (learned) unconditional model that records the class $i$ responsible for killing the backward process initialised in $x$.

\begin{algorithm}
\caption{Classification}
   \label{alg:classification}
\begin{algorithmic}
   \STATE {\bfseries Input:} state $x$ to be classified, classes $i=1,\ldots,K$, estimates $\widehat{\Omega^i}$ of  class supports $\Omega^i$, learned backward model $Y^{\hat{\theta}}$ for $\widehat{\Omega_\varepsilon} = \bigcup_{i=1}^K (\widehat{\Omega^i})_\varepsilon$ based on Algorithm \ref{alg:gen_nothing_known}
   \IF{$x \notin \widehat{\Omega_\varepsilon}$}
    \STATE Initialize $Y^{\hat{\theta}}_0 = x$
    \REPEAT
        \STATE Simulate next step of $Y^{\hat{\theta}}$ 
    \UNTIL{$Y^{\hat{\theta}}_t \in (\widehat{\Omega^i})_\varepsilon$ for some $i \in \{1,\ldots,K\}$}
   \ELSE
    \STATE Find $i \in \{1,\ldots,K\}$ s.t.\ $i \in \argmin_{i=1,\ldots,K} d(x,\widehat{\Omega^i})$
   \ENDIF
   \STATE {\bfseries{Output:}} class $i$
\end{algorithmic}
\end{algorithm}

The presented conditional diffusion model thus naturally enables classification through transfer learning. Refinements are also conceivable. 
For instance, by starting several runs from $x$ we can estimate the conditional class distribution, which allows the construction of statistical tests.
Analogously to the case of anomaly detection, the average lifetime $\overline{\zeta}$ can then be regarded as a measure of reliability.

\subsection{General transfer learning}
The idea of transfer learning presented in the previous section can be applied more generally. 
For instance, if the classes are not predefined, but clusters are to be learned from the data, the typical approach is to perform clustering with known algorithms  
only on the support of $\alpha$, which is a significantly lower-dimensional task. Similar approaches may be possible for tasks in Reinforcement Learning.

\section{Conclusion}
\label{sec:discussion}
We have introduced a novel class of diffusion models grounded in Doob's $h$-transform. 
This theoretical framework provides a  general and adaptable foundation for time-homogeneous diffusion models 
that eliminates the necessity to model artificial time dependencies and thereby enhances interpretability. A key innovation in our model is the introduction of the polarity hypothesis, which closely parallels the manifold hypothesis in machine learning. This property enables an intuitive and efficient mechanism for determining the termination of the denoising process, addressing a crucial challenge in implementing our generative model.

Another notable feature of our framework is a methodological simplification for learning the dynamics of the denoising process. Unlike conventional diffusion models that require estimating temporally inhomogeneous dynamics, we must learn a time-\textit{independent} backward drift, which is achieved by a denoising score matching procedure. While this comes at a cost of having to learn the now random simulation time for the generative process as well, we argue that no separate estimation strategy is needed for this purpose, but that a simple explosion criterion for the estimated backward drift along the generated path yields an adaptive termination rule.

The time-homogeneous nature of the model opens up numerous opportunities for practical applications. While its utility will depend on the specific requirements of each use case, the theoretical framework established here provides a robust foundation for experimental exploration, particularly in the context of transfer learning.

\printbibliography

\appendix
\onecolumn
\section{Doob's $h$-transform and lifetime of a diffusion process}\label{app:doob}
The following is based on \cite{chung2005markov}, which provides a level of generality necessary for our purposes. For a more explicit discussion in the univariate case, see also \cite{Borodin}.

We first note that the Green kernel can be found explicitly for the Brownian motion:
\begin{remark}\label{ex:G_r_BM}
Let $Z=W$ be a standard $d$-dimensional Brownian motion, $m$ Lebesgue measure, then
\[		p_t(x,y)=(2\pi t)^{-d/2}\exp\left(-\frac{\lvert x-y\rvert^2}{2t}\right).
\]
	In this case, \citep[p.\ 146]{erdelyi} yields
	\[
		G_r(x,y)=G_r(\lvert x-y\rvert)=(2\pi)^{-d/2}2\left(\frac{\lvert x-y\rvert^2}{2r}\right)^{(2-d)/4}K_{(d-2)/2}
		\left(\lvert x-y\rvert\sqrt{2r}\right),
	\]
	where $K_\nu$ is the modified Bessel
	function of the second kind as defined in \cite{lebedev}, p.\ 109. The function $G_r(\cdot)$ has a pole in $0$ and is decreasing.
\end{remark}

Now we give the remaining proofs:
\begin{proof}[Proof of Proposition \ref{prop:properties_h-transf}]

The finiteness of\ $\zeta$ follows from
 \citep[Theorem 13.50]{chung2005markov}:
		\[
			\PR_x(\zeta<\infty)
			=\int_0^\infty\int\e^{-rt}\frac{1}{h(x)}p_t(x,y)\kappa(dy)dt=\frac{1}{h(x)}\int {G_r(x,y)}\kappa(dy)=1.
		\]
            The second statement follows by using the generator identity
\[\mathbb{A}^h f(x) = \frac{1}{h(x)} \left( \mathbb{A}(h(x) f(x)) - r h(x) f(x) \right),\]
taking into account that $h$ is $r$-harmonic (i.e., $(\mathbb{A}-r)h(x)=0$) outside the support of $\beta$.

The third statement follows from \citep[Theorem 13.39]{chung2005markov} (after correction of an obvious typo), where the special case $x=x_0$ yields 
\[
    		    \PR_{x_0}(Z^h_{\zeta-}\in dy)=\frac{G_r(x_0,y)}{h(x_0)}\kappa(dy)=G_r(x_0,y)\kappa(dy)=\beta(dy).
\]
\end{proof}

\begin{proof}[Proof of Proposition \ref{prop:time_inversion}]
The first claim holds by combining  \citep[Theorem 13.34]{chung2005markov} with Proposition \ref{prop:properties_h-transf}. The other two are direct consequences of Proposition \ref{prop:properties_h-transf}.
\end{proof}

\section{Stochastic control}\label{app:stoch_control}
For the proof of Proposition \ref{prop:stoch_control}, we start with the following
\begin{lemma}\label{lem:PDE_contr}
    Let $h$ be as above and $g(x)\coloneqq-\log h(x)$. Then, outside of the support of $\beta$, it holds that
    \[\mathbb A g(x)+r-\frac{1}{2}\nabla g(x)^\top\sigma(x)\sigma^\top(x)\nabla g(x)=0.\]
\end{lemma}
\begin{proof}
    Outside of $\supp \kappa$, it holds that $\mathbb A h(x)-rh(x)=0$, and a straightforward calculation shows that
    \[
        \mathbb A g(x)=-r+\frac{1}{2}\nabla g(x)^\top\sigma(x)\sigma^\top(x)\nabla g(x).
    \]
\end{proof}
Now, we follow the standard verification approach to stochastic stopping and control problems, see \cite{oksendal2019stochastic}, Chapter 5, where we assume the appropriate standard regularity assumptions for the objects without explicitly mentioning them. 
Outside of $\supp \kappa$, using the notation 
\[
\mathbb{A}^u v(x) = \langle \sigma(x) u + b(x), \nabla v(x) \rangle + \frac{1}{2} \operatorname{Tr}\left[\Sigma(x) \nabla^2 v(x)\right],
\]
where $\Sigma(x) = \sigma(x)\sigma(x)^\top$, it holds
\begin{align*}
    \inf_u \left(\mathbb{A}^u v(x) + k(u)\right) 
    =& \langle b(x), \nabla v(x) \rangle + \frac{1}{2} \operatorname{Tr}\left[\Sigma(x) \nabla^2 v(x)\right] + r + \inf_u \left( \langle \sigma(x) u, \nabla v(x) \rangle + \frac{1}{2} \|u\|^2 \right) \\
    =& \langle b(x), \nabla v(x) \rangle + \frac{1}{2} \operatorname{Tr}\left[\Sigma(x) \nabla^2 v(x)\right] + r - \frac{1}{2} \|\sigma(x)^\top \nabla v(x)\|^2
\end{align*}
with minimiser $u = -\sigma(x)^\top \nabla v(x)$. Using $v(x)=g(x)=-\log h(x)$ and Lemma \ref{lem:PDE_contr}, we see that $g(x)$ fulfills the HJB equation for problem \eqref{eq:control}.
Following the steps in the verification procedure in \cite{oksendal2019stochastic}, p.\ 92ff., it is easily seen that this solution to the HJB equation is indeed a solution to \eqref{eq:control}, proving the first part of Proposition \ref{prop:stoch_control}.

Now, we study the connection to the KL divergence and write $Y=Z^h$. Note that
\[
    \frac{d\PR^{Z^u}}{d\PR^Y}\Bigg|_{\mathcal F_\zeta}=\exp\left(\frac{1}{2}\int_0^\zeta \|u_t\|^2dt-M_\zeta\right),
\]
for some (local) martingale $M$ with\ $M_0=0$.
For the killed process $Y$, by the definition of the $h$-transform,
\[
    \frac{d\PR^{Z}}{d\PR^{Y}}\Bigg|_{\mathcal F_\zeta}=\frac{\e^{r\zeta}h(Z_0)}{h(Z_\zeta)}.
\]
This yields 
\begin{align*}
    \log \frac{d\PR^{Z^u}}{d\PR^{Y}}\Bigg|_{\mathcal F_\zeta}&=\log \frac{d\PR^{Z^u}}{d\PR^Z}\frac{d\PR^{Z}}{d\PR^{Y}}\Bigg|_{\mathcal F_\zeta}=\frac{1}{2}\int_0^\zeta \|u_t\|^2dt+r\zeta+\log h(Z^u_0)-\log h(Z^u_\zeta)+M_\zeta.
\end{align*}
 We obtain that 
\[
    \mathrm{KL}(\PR_x^{Z^u}\big|_{\mathcal F_\zeta}\,\Vert\,\PR_x^{Y}\big|_{\mathcal F_\zeta})=\Ex_x\log \frac{d\PR^{Z^u}}{d\PR^{Z}}\Bigg|_{\mathcal F_\zeta}=J(u,x)+\log h(x)\geq v(x)+\log h(x)=0
\]
with equality for $u=u^*$, so that the KL divergence is the variational gap in problem \eqref{eq:control}. 

\section{Score matching and manifold learning}\label{app:score}
\begin{proof}[Proof of Proposition \ref{prop:score}]
It holds that
\begin{align}\label{eq:zerleg}
    \mathcal{L}_{\text{ex}}(\theta) &= \Ex\Big[\int_{0}^{\zeta_\varepsilon} \lVert \nabla_y \log \cev{h}(\cev{Z^h_t}) - s_\theta(\cev{Z^h_t}) \rVert^2 \, d{t} \Big]\nonumber\\
    &= \Ex\Big[\int_{0}^{\zeta_\varepsilon} \lvert s_\theta(\cev{Z^h_t}) \rvert^2 \, dt \Big] - 2\Ex\Big[\int_{0}^{\zeta_\varepsilon} \langle \nabla_y \log \cev{h}(\cev{Z^h_t}), s_\theta(\cev{Z^h_t}) \rangle \, dt \Big] + C,
\end{align}
where $C$ is independent of $\theta$. 
Let $\tilde{\beta}(dy) = \PR_\alpha(Z^h_{\zeta-} \in dy)$. Since the time reversal of the $h$-transform $Z^{\cev{h}}$ started in $\tilde{\beta}$ is equal in law to the $h$-transform $Z^h$ started in $\alpha$, we obtain from Proposition \ref{prop:time_inversion} that
\[ \frac{G_r(x,y)}{\cev{h}(x)} \, \tilde{\beta}(dx) =  h(y)\frac{G_r(y,x)}{h(y)\cev{h}(x)}\, \tilde{\beta}(dx) = h(y) \,\PR_y(Z^h_{\zeta-} \in dx).\]
Using this and assuming sufficient integrability properties that allow the application of Fubini's theorem and pulling the derivative inside the integral, we can calculate for the second term as follows:
\begin{align*}
    \Ex\Big[\int_{0}^\zeta \big\langle \nabla_y \log \cev{h}(\cev{Z^h_t}),s_\theta(\cev{Z^h_t})\big\rangle \, dt \Big] 
    &= \int \int  \langle \nabla_y \log \cev{h}(y),s_\theta(y)\rangle \frac{\cev{h}(y)}{\cev{h}(x)}G_r(x,y)\, m(d y)\, \tilde{\beta}(dx)\\
    &= \int \langle \nabla_y \log \cev{h}(y),s_\theta(y)\rangle \cev{h}(y) h(y)\, m(d y) \\
    &= \int \langle \nabla_y \cev{h}(y),s_\theta(y)\rangle h(y)\, m(d y) \\
    &= \int\int \langle \nabla_y \frac{G_r(y,z)}{h(z)}, s_\theta(y) \rangle h(y)\, m(d y)\, \alpha(d z) \\
    &= \int\int \langle \nabla_y \log G_r(y,z), s_\theta(y) \rangle\frac{h(y)}{h(z)}G_r(y,z)\, m(d y)\, \alpha(d z) \\
     &= \int\int \langle \nabla_y \log G_r(y,z), s_\theta(y) \rangle\frac{h(y)}{h(z)}G_r(z,y)\, m(d y)\, \alpha(d z) \\
    &= \int \Ex_z\Big[\int_0^\zeta \langle \nabla_{Z^h_t} \log G_r(Z^h_t,Z^h_0), s_\theta(Z^h_t) \rangle  d{t}\Big] \, \alpha(dz)\\
    &= \Ex_\alpha\Big[\int_{0}^\zeta  \langle\nabla_{Z^h_t} \log G_r(Z^h_t,Z^h_0) , s_\theta(Z^h_t)\rangle \, dt \Big].
\end{align*}
This gives 
\begin{align*} 
&\Ex\Big[\int_{0}^{\zeta_\varepsilon} \big\langle \nabla_y \log \cev{h}(\cev{Z^h_t}),s_\theta(\cev{Z^h_t})\big\rangle \, dt \Big]\\
&\,= \Ex\Big[\int_{0}^{\zeta} \big\langle \nabla_y \log \cev{h}(\cev{Z^h_t}),s_\theta(\cev{Z^h_t})\big\rangle \, dt \Big] - \Ex\Big[\int_{\zeta_\varepsilon}^{\zeta} \big\langle \nabla_y \log \cev{h}(\cev{Z^h_t}),s_\theta(\cev{Z^h_t})\big\rangle \, dt \Big]\\ 
&\,= \Ex\Big[\int_{0}^\zeta  \langle\nabla_{Z^h_t} \log G_r(Z^h_t,Z^h_0) , s_\theta(Z^h_t)\rangle \, dt \Big] - \Ex\Big[\int_{0}^{\sigma_\varepsilon} \big\langle \nabla_y \log \cev{h}({Z^h_t}),s_\theta({Z^h_t})\big\rangle \, dt \Big]\\
&\,= \Ex\Big[\int_{\sigma_\varepsilon}^{\zeta}  \langle\nabla_{Z^h_t} \log G_r(Z^h_t,Z^h_0) , s_\theta(Z^h_t)\rangle \, dt \Big] + \Ex\Big[\int_{0}^{\sigma_\varepsilon} \big\langle \nabla_{Z^h_t} \log \tfrac{G_r(Z^h_t, Z^h_0)}{\cev{h}({Z^h_t})},s_\theta({Z^h_t})\big\rangle \, dt \Big].
\end{align*}
Moreover, the first term satisfies
\[
    \Ex\Big[\int_{0}^{\zeta_\varepsilon} \lVert s_\theta(\cev{Z^h_t}) \rVert^2 \, dt \Big] = \Ex\Big[\int_{\sigma_\varepsilon}^{\zeta} \lVert s_\theta(Z^h_t) \rVert^2 \, dt \Big].
\]
Substituting these results into \eqref{eq:zerleg}, we obtain
\[
    \mathcal{L}_{\text{ex}}(\theta)  
    = \Ex\Big[\int_{\sigma_\varepsilon}^\zeta \lVert \nabla_{Z^h_t} \log G_r(Z^h_t,Z^h_0) - s_\theta(Z^h_t) \rVert^2 \, dt \Big] - 2\Ex\Big[\int_{0}^{\sigma_\varepsilon} \big\langle \nabla_{Z^h_t} \log \tfrac{G_r(Z^h_t, Z^h_0)}{\cev{h}({Z^h_t})},s_\theta({Z^h_t})\big\rangle \, dt \Big]  + C',
\]
where $C'$ is a constant independent of $\theta$.
\end{proof}

\begin{proof}[Proof of Lemma \ref{lem:char_score}]
Using \eqref{eq:back_h} and Proposition \ref{prop:time_inversion}, we find for any $x \notin \supp \alpha$,
\begin{align*}
\nabla \log \cev{h}(x) = \frac{1}{\cev{h}(x)} \int \nabla_x G_r(x,y) \frac{1}{h(y)} \, \alpha(dy) &= \int \nabla_x \log G_r(x,y) \frac{G_r(x,y)}{\cev{h}(x) h(y)} \,\alpha(dy)\\
&= \E_x\big[\nabla_x \log G_r(x, Z^{\cev{h}}_{\zeta-}) \big].
\end{align*}
By Proposition \ref{prop:time_inversion}, the time reversed process $\cev{Z^h}$ has the same distribution as the diffusion process $Z^{\cev{h}}$ started in the forward terminal distribution $\beta_h = \PP_\alpha(Z^{h}_{\zeta-} \in \cdot)$. For any Borel sets $A,B$ it follows
\begin{align*} 
\int_B \PP_\alpha(Z^h_0 \in A \mid Z^h_{\zeta_-} = x) \, \PP_\alpha(Z^h_{\zeta_-} \in dx) &= \PP_\alpha(Z^h_0 \in A, Z^h_{\zeta_-} \in B) \\
&= \PP_{Z^{\cev{h}}_0 \sim \beta}(Z^{\cev{h}}_{\zeta_-} \in A, Z^{\cev{h}}_0 \in B)\\ 
&= \int_B \PP(Z^{\cev{h}}_{\zeta_-} \in A \mid Z^{\cev{h}}_0 = x)\, \beta_h(dx)\\ 
&= \int_B \PP(Z^{\cev{h}}_{\zeta_-} \in A \mid Z^{\cev{h}}_0 = x)\, \PP_\alpha(Z^h_{\zeta_-} \in dx),
\end{align*}
and thus, $\PP(Z^{\cev{h}}_{\zeta_-} \in \diff{y} \mid Z^{\cev{h}}_0 = x)$ is a regular conditional probability for $Z_0^h$ given $Z^h_{\zeta_-}$ w.r.t.\ $\PP_\alpha$. Consequently, uniqueness of regular conditional probabilities gives
\[\PP_\alpha(Z^h_0 \in dy \mid Z^h_{\zeta_-} = x) = \PP(Z^{\cev{h}}_{\zeta_-} \in d y \mid Z^{\cev{h}}_0 = x), \quad \text{for } \PP_\alpha(Z^h_{\zeta_-} \in \cdot)\text{-a.e. } x.\] This together with \eqref{eq:score_char1} and symmetry of $G_r$ yields \eqref{eq:score_char2}. 
\end{proof}

\end{document}